\documentclass[11pt]{article}

\usepackage{fullpage}
\usepackage{mathtools,amssymb,amsthm}
\usepackage[round,authoryear]{natbib}
\usepackage{tikz}
\usetikzlibrary{positioning,shapes.geometric}
\usepackage[hidelinks]{hyperref}

\newcommand{\Kb}{\mathbf{K}}
\newcommand{\Qb}{\mathbf{Q}}
\newcommand{\Tb}{\mathbf{T}}
\newcommand{\Vb}{\mathbf{V}}
\newcommand{\Wb}{\mathbf{W}}
\newcommand{\RR}{\mathbb{R}}
\DeclareMathOperator{\tr}{tr}

\newtheorem{theorem}{Theorem}[section]
\newtheorem{lemma}[theorem]{Lemma}
\newtheorem{corollary}[theorem]{Corollary}
\theoremstyle{definition}
\newtheorem{definition}[theorem]{Definition}
\numberwithin{equation}{section}

\title{Deep Tensor Network}
\author{Yifan Zhang}
\date{}

\begin{document}
\maketitle

\begin{abstract}
The Deep Tensor Network is a family of sequence operators built from contractions of queries, keys, and values. Tensor Attention mixes tokens through the Gram operators $Q(K^\top K)Q^\top$ and $K(Q^\top Q)K^\top$; Tensor Interaction mixes channels through operators derived from $C=Q^\top K$. Associativity lets us evaluate these products exactly without forming $n\times n$ token matrices. For query/key width $d$ and value width $d_v$, the Gram constructions use $O(d^2+dd_v)$ work and summary state per causal token, or $O(d^2)$ when $d_v=O(d)$. We derive prefix recurrences and distinguish trace, diagonal, and row normalization, including the conditions under which the weights are nonnegative and row-stochastic. The unmasked token operators are positive semidefinite with rank at most $d$, but their entries can be negative. We also give exact evaluations and costs for Hadamard and matrix-exponential variants.
\end{abstract}

\section{Introduction}

The \emph{Deep Tensor Network} uses query/key Gram statistics to mix tokens and channels. We study the contractions that define these operators, how to evaluate them without forming a token-by-token matrix, and which properties survive normalization and causal evaluation.

The scaled dot-product attention introduced by \citet{DBLP:conf/nips/VaswaniSPUJGKP17} has quadratic cost in sequence length, which limits its use on long sequences. Linear attention methods reduce this cost while retaining modeling performance~\citep{katharopoulos2020transformers}.

For finite-dimensional vector spaces $\mathcal{X}$ and $\mathcal{Y}$, the tensor product $\mathcal{X} \otimes \mathcal{Y}$ is a space in which to represent higher-order interactions between feature coordinates. Its universal property and the associated relations between linear maps give the algebraic background for the contractions below.

We build two operators from these contractions. \emph{Tensor Attention} mixes tokens through Gram operators, and \emph{Tensor Interaction} mixes channels using query/key cross-statistics. Both depend on the input and can be evaluated by reassociating their matrix products. The token Gram operators are symmetric and low-rank.

We give the following analysis:
\begin{itemize}
    \item We relate tensor products and their contractions to attention operators.
    \item We derive exact evaluations of the Tensor Attention Gram operators and establish their rank, normalization, and computational costs.
    \item We define the Tensor Interaction channel operator and derive causal prefix recurrences for both constructions.
\end{itemize}

\section{Preliminaries}
\label{sec:preliminaries}

\subsection{Scaled Dot-Product Attention}

For self-attention, let $\Qb,\Kb\in\RR^{n\times d}$ and $\Vb\in\RR^{n\times d_v}$, with rows $\mathbf q_i^\top$, $\mathbf k_i^\top$, and $\mathbf v_i^\top$. Queries and keys must have the same channel dimension for their dot product to be defined. Scaled dot-product attention~\citep{DBLP:conf/nips/VaswaniSPUJGKP17} is
\[
\operatorname{Attention}(\Qb,\Kb,\Vb)
=\operatorname{softmax}_{\mathrm{row}}\!\left(\frac{\Qb\Kb^\top}{\sqrt d}\right)\Vb.
\]
The $i$th output row is
\[
\mathbf o_i^\top=\sum_{j=1}^{n}a_{ij}\mathbf v_j^\top,
\qquad
 a_{ij}=\frac{\exp(\mathbf q_i^\top\mathbf k_j/\sqrt d)}
 {\sum_{\ell=1}^{n}\exp(\mathbf q_i^\top\mathbf k_\ell/\sqrt d)}.
\]
The denominator normalizes over keys for a fixed query. Explicitly forming the score matrix costs $O(n^2d)$ time and $O(n^2)$ score storage; multiplying the weights by values costs $O(n^2d_v)$ time.

\subsection{Multi-Head Attention}

For an input $X\in\RR^{n\times d_{\mathrm{model}}}$, define
\[
\begin{aligned}
\operatorname{head}_a&=\operatorname{Attention}(X\Wb_a^Q,X\Wb_a^K,X\Wb_a^V),\\
\operatorname{MultiHead}(X)&=\operatorname{Concat}(\operatorname{head}_1,\ldots,\operatorname{head}_h)\Wb^O,
\end{aligned}
\]
where $\Wb_a^Q,\Wb_a^K\in\RR^{d_{\mathrm{model}}\times d}$, $\Wb_a^V\in\RR^{d_{\mathrm{model}}\times d_v}$, and $\Wb^O\in\RR^{hd_v\times d_{\mathrm{model}}}$. A common choice is $d=d_v=d_{\mathrm{model}}/h$; $d_{\mathrm{model}}=512$ and $h=8$ then give $d=d_v=64$.

\paragraph{Notation.}
We work over $\RR$ unless stated otherwise; the complex-valued extension uses $^{\mathrm H}$ for conjugate transpose and keeps the matrix dimensions unchanged. We use $d$ for the per-head query/key dimension and $d_v$ for the value dimension. Costs are per head and exclude input projections. The vectors $\mathbf q_i,\mathbf k_i,\mathbf v_i$ are columns, and outputs are rows $\mathbf o_i^\top$. For a vector $x$, $\operatorname{Diag}(x)$ is the diagonal matrix with diagonal $x$; for a matrix $M$, $\operatorname{diag}(M)$ is its diagonal vector. $I_m$ and $\mathbf 1_m$ denote the identity and all-ones vector of dimension $m$. Unless stated otherwise, $\varepsilon>0$ is a fixed denominator regularizer.

\subsection{Kernelized Attention}

Given a feature map $\phi:\RR^d\to\RR^r$, define $\kappa(q,k)=\phi(q)^\top\phi(k)$. If $\kappa(\mathbf q_i,\mathbf k_j)\ge0$ and $s_i=\sum_j\kappa(\mathbf q_i,\mathbf k_j)>0$, kernel attention is
\[
\mathbf o_i^\top=\frac{\sum_{j=1}^n\kappa(\mathbf q_i,\mathbf k_j)\mathbf v_j^\top}{s_i}.
\]
For example, set $u(x)=x/\|x\|_2$ for $x\ne0$ and $u(0)=0$. The feature map $\phi(x)=(1,u(x)^\top)^\top$ gives the exact kernel
\[
\kappa(q,k)=1+u(q)^\top u(k)\ge0.
\]
The approximation $e^z=1+z+O(z^2)$ applies locally to the normalized score $z=u(q)^\top u(k)$, not to arbitrary unnormalized scaled dot products. Antipodal unit vectors can give $\kappa=0$, so nonnegativity alone does not ensure a nonzero denominator. Replacing $s_i$ by $s_i+\varepsilon$ defines a regularized operator whose weights sum to $s_i/(s_i+\varepsilon)$, rather than exactly one.

\subsection{Linear Attention}

Write $F_Q,F_K\in\RR^{n\times r}$ for the rowwise feature matrices. Associativity gives~\citep{katharopoulos2020transformers}
\[
\operatorname{KernelAttention}
=\operatorname{Diag}\!\bigl(F_Q(F_K^\top\mathbf1_n)\bigr)^{-1}
 F_Q(F_K^\top\Vb),
\]
provided all row sums are nonzero. Excluding feature evaluation, the numerator costs $O(nrd_v)$ and the normalizer costs $O(nr)$; causal evaluation maintains $O(rd_v+r)$ running statistics. Performers~\citep{choromanski2021rethinking} approximate the exponential kernel with positive random features, with the factor $d^{-1/2}$ absorbed into the query/key scaling. Numerical stability requires positive denominators and control of floating-point range.

\section{Tensor Category}

\subsection{Tensor Product}

\begin{definition}
Let $V$ and $W$ be vector spaces over a field $\mathbb{F}$. Their \emph{tensor product}, denoted by $V \otimes W$, is defined as a vector space equipped with a bilinear map
$$
\otimes : V \times W \to V \otimes W,
$$
which satisfies the following \emph{universal property}: for every vector space $Z$ and every bilinear map $\tau : V \times W \to Z$, there exists a unique linear map $\tilde{\tau} : V \otimes W \to Z$ such that
$$
\tau(v,w) = \tilde{\tau}(v \otimes w) \quad \text{for all } v \in V,\; w \in W.
$$
\end{definition}

Every bilinear map from $V \times W$ to a vector space $Z$ therefore factors uniquely through the tensor product.

\subsection{Uniqueness of the Tensor Product}

The universal property determines the tensor product up to canonical isomorphism.

\begin{theorem}[Uniqueness Theorem of Tensor Products]
Let $U$ and $T$ be two tensor products of $V$ and $W$, each satisfying the universal property. Then, there exists a canonical isomorphism between $U$ and $T$.
\end{theorem}

\begin{proof}
Since both $U$ and $T$ satisfy the universal property, there exist unique linear maps $\alpha: U \to T$ and $\beta: T \to U$ making the corresponding diagrams commute. Uniqueness implies that $\alpha\circ\beta = \mathrm{id}_T$ and $\beta\circ\alpha = \mathrm{id}_U$, so $\alpha$ and $\beta$ are inverses of each other. Hence, $U$ and $T$ are canonically isomorphic.
\end{proof}

\begin{corollary}
The tensor product is symmetric; that is, $V \otimes W$ is canonically isomorphic to $W \otimes V$.
\end{corollary}

\begin{proof}
The bilinear map $(v,w)\mapsto w\otimes v$ induces a unique linear map $F:V\otimes W\to W\otimes V$ with $F(v\otimes w)=w\otimes v$. Reversing the roles of $V$ and $W$ gives $G:W\otimes V\to V\otimes W$. Both compositions fix every pure tensor, and pure tensors span the tensor products. Thus $GF$ and $FG$ are identity maps.
\end{proof}

\begin{lemma}[\citet{bradley2020interface}]
Let $V$ and $W$ be finite-dimensional vector spaces. Then, there is an isomorphism
$$
\operatorname{End}(V \otimes W) \cong \operatorname{End}(V) \otimes \operatorname{End}(W).
$$
\end{lemma}

\begin{proof}
For finite-dimensional spaces $A$ and $B$, one has $\operatorname{Hom}(A,B) \cong B \otimes A^*$. Hence,
$$
\operatorname{End}(V \otimes W) \cong (V \otimes W) \otimes (V \otimes W)^* \cong V \otimes W \otimes V^* \otimes W^*.
$$
Rearranging the factors yields
$$
V \otimes V^* \otimes W \otimes W^* \cong \operatorname{End}(V) \otimes \operatorname{End}(W).
$$
\end{proof}

\subsection{Partial Trace}

\begin{definition}[Partial trace~\citep{bradley2020interface}]
Let $V$ and $W$ be finite-dimensional real inner-product spaces with orthonormal bases $\{e_i\}_{i=1}^m$ and $\{f_j\}_{j=1}^n$. In these bases, write an operator on $V\otimes W$ as
\[
T=\sum_{k,i=1}^m\sum_{l,j=1}^n
 T_{kl,ij}\,(e_k\otimes f_l)(e_i\otimes f_j)^\top.
\]
Its partial traces are
\[
\operatorname{Tr}_W(T)=\sum_{k,i=1}^m\left(\sum_{j=1}^nT_{kj,ij}\right)e_ke_i^\top,
\qquad
\operatorname{Tr}_V(T)=\sum_{l,j=1}^n\left(\sum_{i=1}^mT_{il,ij}\right)f_lf_j^\top.
\]
For complex inner-product spaces, replace the transposes in the outer products by conjugate transposes.
\end{definition}

\subsection{Tensor Products, Reshaping, and Contraction}

A matrix is a linear map between specified vector spaces, not itself a vector space. For $\Qb,\Kb\in\RR^{n\times d}$, the Kronecker product $\Qb\otimes\Kb$ is an $n^2\times d^2$ matrix representing a map from $\RR^d\otimes\RR^d$ to $\RR^n\otimes\RR^n$. It is distinct from the contracted matrix $A=\Qb\Kb^\top\in\RR^{n\times n}$.

Define the uncontracted order-eight array
\[
\mathcal U_{ia,jb,kc,\ell e}=Q_{ia}K_{jb}Q_{kc}K_{\ell e}.
\]
It has four token indices of size $n$ and four channel indices of size $d$. Grouping its token indices as rows and channel indices as columns yields the $n^4\times d^4$ matrix $\Qb\otimes\Kb\otimes\Qb\otimes\Kb$. Contracting the first two channel indices and the last two gives an order-four array:
\[
\mathcal R_{ij,k\ell}
=\sum_{a=1}^d\sum_{c=1}^d\mathcal U_{ia,ja,kc,\ell c}
=A_{ij}A_{k\ell}.
\]
There are different matricizations of this same array:
\[
(A\otimes A)_{(i,k),(j,\ell)}=A_{ij}A_{k\ell},
\qquad
[\operatorname{vec}(A)\operatorname{vec}(A)^\top]_{(i,j),(k,\ell)}=A_{ij}A_{k\ell}.
\]
Here $\operatorname{vec}$ stacks columns, and the paired labels on its outer product refer to the corresponding entries of $A$. The two matricizations are related by a reshuffling of indices, not by ordinary matrix equality. Reshuffling preserves the entries but need not preserve matrix rank; constructing this outer product from $A\otimes A$ is not a contraction. The usual identity
\[
\operatorname{vec}(BXC^\top)=(C\otimes B)\operatorname{vec}(X)
\]
holds for conformable matrices with the same column-stacking convention.

Summing over a repeated token index contracts the array:
\[
[\Tb_{\Qb}]_{ik}=\sum_{j=1}^n\mathcal R_{ij,kj}
=\sum_{j=1}^n A_{ij}A_{kj}=[AA^\top]_{ik},
\qquad
\Tb_{\Kb}=A^\top A.
\]
These are also partial traces of $\psi_A\psi_A^\top$, where $\psi_A=\sum_{i,j}A_{ij}e_i\otimes e_j$: tracing out the second factor gives $AA^\top$, and tracing out the first gives $A^\top A$. This contraction loses information: the smaller matrix does not determine the tensor.

\subsection{Tensor Network Diagrams}

Each open edge represents a free index and each connecting edge represents summation over a shared index. First contract the channel index $a$ to obtain $A_{ij}=\sum_aQ_{ia}K_{ja}$:
\begin{center}
\begin{tikzpicture}
  \node[draw,circle] (Q) at (0,0) {$Q$};
  \node[draw,circle] (K) at (3,0) {$K$};
  \draw (Q) -- node[above] {$a\;(d)$} (K);
  \draw (Q) -- ++(0,1) node[above] {$i\;(n)$};
  \draw (K) -- ++(0,1) node[above] {$j\;(n)$};
\end{tikzpicture}
\end{center}
Next contract the \emph{second} index of each copy of $A$. The result is $[\Tb_{\Qb}]_{ik}=\sum_jA_{ij}A_{kj}$, not $[A^2]_{ik}=\sum_jA_{ij}A_{jk}$:
\begin{center}
\begin{tikzpicture}
  \node[draw,circle] (A1) at (0,0) {$A$};
  \node[draw,circle] (A2) at (3,0) {$A$};
  \draw (A1) -- node[above] {$j\;(n)$} (A2);
  \draw (A1) -- ++(0,1) node[above] {$i\;(n)$};
  \draw (A2) -- ++(0,1) node[above] {$k\;(n)$};
\end{tikzpicture}
\end{center}
The index labels specify the connected axes. Moving a node in the diagram does not transpose its entries.

\section{Tensor Attention}

Throughout this section, $\Qb,\Kb\in\RR^{n\times d}$ and $\Vb\in\RR^{n\times d_v}$. Tensor Attention mixes tokens through query/key Gram operators.

\subsection{Definition and Kernel Interpretation}

Set $A=\Qb\Kb^\top$. The two token-space operators are
\[
\begin{aligned}
\Tb_{\Qb}&=AA^\top=\Qb(\Kb^\top\Kb)\Qb^\top\in\RR^{n\times n},\\
\Tb_{\Kb}&=A^\top A=\Kb(\Qb^\top\Qb)\Kb^\top\in\RR^{n\times n}.
\end{aligned}
\]
We take trace normalization as the baseline:
\[
\operatorname{TensorAttention}_{\mathrm{tr},b}(\Qb,\Kb,\Vb)
=\frac{\Tb_b\Vb}{\tr(\Tb_b)+\varepsilon},\qquad b\in\{\Qb,\Kb\}.
\]
The regularizer keeps the expression defined when $A=0$. Row and diagonal normalization give different operators, which we discuss below.

For the $\Qb$ branch, write $S^K=\Kb^\top\Kb$. For fixed keys, the feature map
\[
\phi_{\Kb}(q)=\Kb q\in\RR^n
\]
satisfies
\[
[\Tb_{\Qb}]_{ij}=\mathbf q_i^\top S^K\mathbf q_j
=\langle\phi_{\Kb}(\mathbf q_i),\phi_{\Kb}(\mathbf q_j)\rangle.
\]
Equivalently one may use $(S^K)^{1/2}q\in\RR^d$. This is a data-dependent linear (Mahalanobis) kernel in the queries, not the degree-two polynomial kernel $(q^\top q')^2$. As a function of $\Qb$ and $\Kb$ jointly, it is quartic.

\subsection{Positive Semidefiniteness and Its Limits}

\begin{lemma}\label{lem:psd}
For any real matrix $A$, $AA^\top$ and $A^\top A$ are symmetric positive semidefinite (PSD), with rank at most $\operatorname{rank}(A)$. In particular, both token operators above have rank at most $\min(n,d)$ and have nonnegative diagonal entries.
\end{lemma}
\begin{proof}
For any $x$, $x^\top AA^\top x=\|A^\top x\|_2^2\ge0$. The other case is identical, and the rank bound follows from the rank of a product.
\end{proof}

For a PSD matrix $T$, $\|T\|_2\le\tr(T)$. Thus the baseline satisfies
\[
\left\|\frac{T}{\tr(T)+\varepsilon}\right\|_2\le1,
\qquad
\left\|\frac{T\Vb}{\tr(T)+\varepsilon}\right\|_F\le\|\Vb\|_F.
\]
The bound holds for fixed queries and keys.

PSD does not imply entrywise nonnegativity. For example, $\Qb=(1,-1)^\top$ and $\Kb=(1,0)^\top$ give
\[
\Tb_{\Qb}=\begin{pmatrix}1&-1\\-1&1\end{pmatrix},
\qquad \Tb_{\Qb}\mathbf1_2=0.
\]
Consequently, arbitrary real inputs do not yield probability weights, and dividing by row sums may be undefined. Even when row sums are nonzero, signed coefficients summing to one are not a stochastic matrix. The raw Gram kernel is symmetric and low-rank, whereas a rowwise softmax of dot products can be asymmetric and full-rank.

\subsection{Efficient Evaluation and Trace Normalization}

Define the summaries
\[
S^K=\Kb^\top\Kb\in\RR^{d\times d},\qquad
B=\Qb^\top\Vb\in\RR^{d\times d_v},\qquad
S^Q=\Qb^\top\Qb\in\RR^{d\times d}.
\]
The exact numerator is
\[
\Tb_{\Qb}\Vb=(\Qb S^K)B.
\]
Building $S^K$ and $S^Q$ costs $O(nd^2)$; building $B$ costs $O(ndd_v)$. Evaluate each row as $(\mathbf q_i^\top S^K)B$, costing $O(d^2+dd_v)$ per row. This ordering avoids first forming $S^KB$, which would introduce a $d^2d_v$ matrix-multiplication term. Total time is $O(n(d^2+dd_v))$ with $O(d^2+dd_v)$ summary storage, excluding input and output arrays. These are forward-pass bounds.

\begin{lemma}\label{lem:trace_sum}
The two token operators have the same trace:
\[
\tr(\Tb_{\Qb})=\tr(\Tb_{\Kb})=\|\Qb\Kb^\top\|_F^2
=\tr(S^KS^Q)=\sum_{a,b=1}^d S^K_{ab}S^Q_{ab}.
\]
\end{lemma}
\begin{proof}
Cyclic invariance gives $\tr(\Qb S^K\Qb^\top)=\tr(S^KS^Q)$. Since both summaries are symmetric, this trace is their Frobenius inner product. The remaining identities follow from $\tr(AA^\top)=\tr(A^\top A)=\sum_{i,j}A_{ij}^2$.
\end{proof}
Once the summaries are available, the trace takes $O(d^2)$ work. Although it is nonnegative in exact arithmetic, the sum can cancel in finite precision; if cancellation is severe, use a suitable PSD factorization or higher precision.

\paragraph{Low-rank approximations.}
If $S^K\approx R^\top R$ with $R\in\RR^{r\times d}$, applying the approximate numerator rowwise costs $O(rd+dd_v)$. Storing $B$ still costs $O(dd_v)$, so the numerator summary storage is $O(rd+dd_v)$. Maintaining $R$ requires a sketch or factor-update algorithm, whose update cost and approximation error depend on that algorithm. Trace normalization additionally requires $S^Q$ or another method for estimating the trace.

\subsection{Row and Diagonal Normalization}

Let $z^Q=\Qb^\top\mathbf1_n$. The row sums and diagonal entries are different:
\[
s_i=\mathbf q_i^\top S^K z^Q,
\qquad
\delta_i=\mathbf q_i^\top S^K\mathbf q_i.
\]
If every $s_i\ne0$, row normalization gives
\[
O_{\mathrm{row}}=\operatorname{Diag}(s)^{-1}(\Qb S^K)B.
\]
It has row weights summing to one, but these are probabilities only when $\Tb_{\Qb}$ is entrywise nonnegative. Diagonal normalization instead gives
\[
O_{\mathrm{diag}}=\operatorname{Diag}(\delta+\varepsilon\mathbf1_n)^{-1}(\Qb S^K)B.
\]
This is defined for all real inputs because $\delta_i\ge0$, but it need not satisfy the spectral bound of trace normalization. The $\Kb$ branch follows by exchanging queries and keys, including replacing $B$ by $\Kb^\top\Vb$.

\paragraph{Nonnegative features.}
To obtain nonnegative Gram weights, transform both queries and keys: let $f:\RR^d\to\RR_{>0}^r$, and form $Q_+,K_+\in\RR^{n\times r}$ from its rowwise values. Then
\[
T_+=Q_+(K_+^\top K_+)Q_+^\top
\]
is PSD and entrywise strictly positive for $n,r\ge1$. Hence $\operatorname{Diag}(T_+\mathbf1_n)^{-1}T_+$ is row-stochastic. Making only the keys nonnegative is not enough: in the preceding counterexample $K$ is already nonnegative and the queries are signed. Nonnegative rather than strictly positive features also permit zero rows; those require an explicit fallback. Adding $\varepsilon$ only to the denominator changes the row sum, so the resulting weights are no longer exactly normalized. Feature evaluation has its own cost; the summary bounds above then use $r$ in place of $d$.

An optional directional branch can use $R_+=Q_+K_+^\top$. With $\operatorname{RowNorm}(M)=\operatorname{Diag}(M\mathbf1_n)^{-1}M$, define
\[
O_{\mathrm{mix}}=
\bigl[\alpha\operatorname{RowNorm}(R_+)+(1-\alpha)\operatorname{RowNorm}(T_+)\bigr]\Vb,
\qquad \alpha\in[0,1].
\]
Both branches have positive row sums, so their convex mixture is row-stochastic. The directional branch uses $Q_+(K_+^\top\Vb)$ and $Q_+(K_+^\top\mathbf1_n)$ without constructing $R_+$.

\subsection{Causality and Streaming Updates}

Assume the representations at position $t$ depend only on input positions at most $t$; for next-token prediction they must also follow the usual target shift. Maintain
\[
\begin{aligned}
S_t^K&=S_{t-1}^K+\mathbf k_t\mathbf k_t^\top,&
S_t^Q&=S_{t-1}^Q+\mathbf q_t\mathbf q_t^\top,\\
B_t&=B_{t-1}+\mathbf q_t\mathbf v_t^\top,&
z_t^Q&=z_{t-1}^Q+\mathbf q_t,
\end{aligned}
\]
starting from zero. $S_t^Q$ is needed for trace normalization and $z_t^Q$ for row normalization. The numerator for the current output is
\[
u_t^\top=(\mathbf q_t^\top S_t^K)B_t
=\sum_{j=1}^t(\mathbf q_t^\top S_t^K\mathbf q_j)\mathbf v_j^\top.
\]
The three causal normalizations are
\[
\begin{aligned}
(\mathbf o_t^{\mathrm{tr}})^\top
 &=\frac{u_t^\top}{\tr(S_t^KS_t^Q)+\varepsilon},\\
(\mathbf o_t^{\mathrm{diag}})^\top
 &=\frac{u_t^\top}{\mathbf q_t^\top S_t^K\mathbf q_t+\varepsilon},\\
(\mathbf o_t^{\mathrm{row}})^\top
 &=\frac{u_t^\top}{\mathbf q_t^\top S_t^K z_t^Q}.
\end{aligned}
\]
The last expression requires a nonzero denominator; with the strictly positive feature construction above it gives convex weights. Per token, the updates and the current readout cost $O(d^2+dd_v)$ time and state, which is $O(d^2)$ when $d_v=O(d)$. A full causal pass with the recurrence costs $O(n(d^2+dd_v))$.

Each output is the \emph{last row} of the operator evaluated on its own prefix. The earlier rows of $\Qb_{\le t}S_t^K\Qb_{\le t}^\top\Vb_{\le t}$ depend on later positions, so computing all its rows at once does not give causal outputs for them. Taking the lower triangle of the full-sequence $\Qb(\Kb^\top\Kb)\Qb^\top$ does not remove its dependence on future keys either. The assembled causal operator is generally neither symmetric nor PSD; the PSD and norm bounds above apply to each fixed-prefix Gram matrix, not to the full lower-triangular map.

\paragraph{Ridge terms, residuals, and windows.}
A ridge term $S_t^K+\lambda I_d$ with $\lambda\ge0$ enters the numerator and every normalizer. For trace normalization the denominator becomes $\tr((S_t^K+\lambda I_d)S_t^Q)+\varepsilon$. At a fixed prefix the added token operator is $\lambda\Qb_{\le t}\Qb_{\le t}^\top$, not $\lambda I_t$. A separate identity residual can instead be written $O_{\mathrm{res}}=(1-\beta)\Vb+\beta O_{\mathrm{tr}}$, $0\le\beta\le1$. For a window of exactly $w$ positions use indices $\max(1,t-w+1),\ldots,t$ in every summary and subtract the contribution at $t-w$ when it leaves. Exact window updates also need access to outgoing inputs, for example a buffer of $O(w(d+d_v))$ entries. These fixed-size summaries rely on the prefix or window structure of the mask.

\subsection{Other Variants and Their Limitations}

\paragraph{Hadamard products.}
The entrywise variant is
\[
H=A\odot A^\top,\qquad H_{ij}=A_{ij}A_{ji}.
\]
The expressions $A\odot A^\top$ and $A^\top\odot A$ are identical, so they do not define distinct query- and key-side operators. Unlike $AA^\top$, $H$ need not be PSD: $A=\left(\begin{smallmatrix}0&1\\1&0\end{smallmatrix}\right)$ gives $H=A$, with eigenvalues $1$ and $-1$ and zero trace. Trace normalization therefore has no analogous norm guarantee. Scaling by $\|H\|_F+\varepsilon$ is defined and bounds its spectral norm, but computing that normalizer has its own cost.

For real inputs,
\[
H_{ij}=(\mathbf q_i\otimes\mathbf k_i)^\top(\mathbf k_j\otimes\mathbf q_j).
\]
Evaluating unnormalized $H\Vb$ through $d^2$-dimensional tensor features takes $O(nd^2d_v)$ time and $O(d^2d_v)$ accumulated value statistics. The same bounds hold for the prefix-masked feature product. They are generally larger than the Gram construction's bounds: per-token cost is cubic in $d$ when $d_v=d$, rather than $O(d^2)$.

\paragraph{Matrix exponential.}
Let $\widehat T=\Tb_{\Qb}/(\tr(\Tb_{\Qb})+\varepsilon)$. Its matrix exponential is different from entrywise exponentiation in softmax. A dense implementation can cost $O(n^3)$, but the low-rank Gram form permits an exact reduction. With
\[
U=\frac{\Qb(S^K)^{1/2}}{\sqrt{\tr(\Tb_{\Qb})+\varepsilon}},\qquad
G=U^\top U,
\]
we have $\widehat T=UU^\top$ and
\[
\exp(\widehat T)\Vb=\Vb+U\psi(G)(U^\top\Vb),
\qquad
\psi(x)=\begin{cases}(e^x-1)/x,&x\ne0,\\1,&x=0.\end{cases}
\]
Here $\psi(G)$ is defined by the spectral decomposition of the PSD matrix $G$. The identity follows from $(UU^\top)^m=UG^{m-1}U^\top$ for $m\ge1$. It costs $O(nd^2+d^3+ndd_v)$ with dense $d\times d$ decompositions. The exponential can have signed entries and is not row-stochastic in general. These are batch bounds.

\paragraph{Complex-valued extension.}
For $\Qb,\Kb\in\mathbb C^{n\times d}$ and $\Vb\in\mathbb C^{n\times d_v}$, set $A=\Qb\Kb^{\mathrm H}$. The Gram operators are $AA^{\mathrm H}$ and $A^{\mathrm H}A$, which are Hermitian PSD. In particular,
\[
AA^{\mathrm H}\Vb=\Qb(\Kb^{\mathrm H}\Kb)(\Qb^{\mathrm H}\Vb),
\qquad \tr(AA^{\mathrm H})=\|A\|_F^2.
\]
Trace normalization retains its norm bound. The distinct Hadamard construction $A\odot A^{\mathrm H}$ is Hermitian but may be indefinite; it is generally not the matrix of squared magnitudes $A\odot\overline A$.

\section{Tensor Interaction}

Tensor Interaction uses a channel-space contraction rather than an $n\times n$ token operator. Take $\Qb,\Kb\in\RR^{n\times d}$ and, initially, $\Vb\in\RR^{n\times d}$. Define
\[
C=\Qb^\top\Kb\in\RR^{d\times d},\qquad
\mathbb T_{\Qb}=CC^\top,\qquad
\mathbb T_{\Kb}=C^\top C.
\]
Both channel operators are symmetric PSD, and their common trace is $\|C\|_F^2$. The trace-normalized query-side output is
\[
\operatorname{TensorInteraction}_{\Qb}(\Qb,\Kb,\Vb)
=\frac{(\Vb C)C^\top}{\|C\|_F^2+\varepsilon}.
\]
The key-side output replaces $(\Vb C)C^\top$ by $(\Vb C^\top)C$. For fixed queries and keys, the normalized channel matrix has spectral norm at most one.

If $d_v\ne d$, introduce projections $P\in\RR^{d_v\times d}$ and $R\in\RR^{d\times d_v}$:
\[
O=\frac{((\Vb P)C)C^\top R}{\|C\|_F^2+\varepsilon}\in\RR^{n\times d_v}.
\]
For arbitrary learned $P$ and $R$, the norm bound acquires a factor of $\|P\|_2\|R\|_2$. Omitting $R$ leaves a $d$-dimensional output rather than a $d_v$-dimensional one.

Building $C$ costs $O(nd^2)$. Evaluating the two value multiplications in the displayed order also costs $O(nd^2)$ when $d_v=d$. Explicitly forming $CC^\top$ would instead cost $O(d^3)$ in addition to applying it. With the projections above, evaluation costs $O(n(d^2+dd_v))$. The learned projections are parameters, separate from the $O(d^2)$ accumulated summary.

\subsection{Causal Channel Mixing}

A full-sequence $C$ depends on future tokens, so using it at every position is bidirectional. For causal representations, use the prefix recurrence
\[
C_t=C_{t-1}+\mathbf q_t\mathbf k_t^\top,\qquad C_0=0,
\]
and, for $d_v=d$, output
\[
\mathbf o_t^\top
=\frac{(\mathbf v_t^\top C_t)C_t^\top}{\|C_t\|_F^2+\varepsilon}.
\]
This evaluates the current row of the channel operator on its own prefix, using $O(d^2)$ time for the update and readout and $O(d^2)$ state. It transforms the current value using the prefix statistic $C_t$; Tensor Attention instead sums past values. For other value dimensions, apply $P$ and $R$ as above. The key-side readout uses $(\mathbf v_t^\top C_t^\top)C_t$.

For a sliding window of exactly $w$ tokens, define
\[
C_t^{(w)}=\sum_{i=\max(1,t-w+1)}^t\mathbf q_i\mathbf k_i^\top.
\]
Update by adding $\mathbf q_t\mathbf k_t^\top$ and, when $t>w$, subtracting $\mathbf q_{t-w}\mathbf k_{t-w}^\top$. A streaming implementation needs the outgoing pair, for example an $O(wd)$ buffer in addition to the $O(d^2)$ summary.

\subsection{Hadamard Channel Variant}

The product $CC^\top$ has entries $\sum_a C_{ia}C_{ja}$, whereas
\[
H_C=C\odot C^\top,\qquad [H_C]_{ij}=C_{ij}C_{ji}.
\]
The two products generally differ. Swapping the Hadamard factors leaves $H_C$ unchanged, so it gives no second branch. This symmetric matrix may be indefinite, and its trace does not reliably bound its scale. We can instead normalize by its Frobenius norm:
\[
O_H=\frac{\Vb H_C}{\|H_C\|_F+\varepsilon}.
\]
Unlike the token-space Hadamard variant, this matrix has size $d\times d$: computing $H_{C_t}$, its Frobenius norm, and the current value product costs $O(d^2)$ per token after updating $C_t$. Its spectral norm after normalization is at most one, but its entries need not be nonnegative.

\section{Related Work}

\paragraph{Efficient Attention.}
Linear attention~\citep{katharopoulos2020transformers} and Performers~\citep{choromanski2021rethinking}, reviewed in Section~\ref{sec:preliminaries}, replace the softmax score with a feature-map kernel between queries and keys, so that attention can be evaluated by reassociating matrix products. Tensor Attention uses the same reassociation, but its token operator is a Gram matrix: it compares pairs of queries (or, in the $\Kb$ branch, pairs of keys) under a data-dependent metric.

\paragraph{Large Language Models.}
Neural language models have grown very large and perform well across many tasks~\citep{devlin2018bert, raffel2020exploring, radford2018improving, radford2019language, brown2020language, OpenAI2023GPT4TR}. GPT-3~\citep{brown2020language} and later models, including Gopher~\citep{rae2021scaling}, PaLM~\citep{chowdhery2022palm}, GLaM~\citep{du2022glam}, Chinchilla~\citep{hoffmann2022training}, Megatron–Turing NLG~\citep{smith2022using}, LaMDA~\citep{thoppilan2022lamda}, OPT~\citep{zhang2022opt}, LLaMA~\citep{touvron2023llama}, PaLM 2~\citep{anil2023palm}, and GPT-4~\citep{OpenAI2023GPT4TR}, show that large autoregressive language models can produce high-quality results without extensive task-specific data collection or parameter updates.

\paragraph{Vision Transformers.}
Convolutional Neural Networks (CNNs) have served as baselines for visual tasks~\citep{DBLP:journals/pieee/LeCunBBH98, DBLP:conf/nips/KrizhevskySH12, DBLP:conf/cvpr/SzegedyLJSRAEVR15, DBLP:conf/cvpr/HeZRS16, DBLP:conf/cvpr/XieGDTH17, DBLP:journals/corr/HowardZCKWWAA17, DBLP:conf/icml/TanL19, DBLP:conf/cvpr/0003MWFDX22}, while Transformers have produced strong results on language tasks~\citep{devlin2018bert, radford2018improving, liu2019roberta, radford2019language, raffel2020exploring, brown2020language, chowdhery2022palm, fedus2022switch, ouyang2022training}. Vision Transformers (ViT), introduced by \citet{DBLP:conf/iclr/DosovitskiyB0WZ21}, have matched or exceeded CNN performance in image recognition, motivating further work on model architectures~\citep{DBLP:conf/iccv/LiuL00W0LG21, DBLP:conf/iccv/WangX0FSLL0021, DBLP:journals/cvm/WangXLFSLLLS22, DBLP:conf/cvpr/DongBCZYYCG22, DBLP:conf/nips/AliTCBDJLNSVJ21, DBLP:journals/corr/abs-2111-07832, DBLP:journals/corr/abs-2103-00112, DBLP:conf/cvpr/Liu0LYXWN000WG22, DBLP:conf/nips/RiquelmePMNJPKH21, DBLP:conf/cvpr/Zhai0HB22, DBLP:conf/nips/DaiLLT21}.

\paragraph{Self-Supervised Learning in Vision Tasks.}
Self-Supervised Learning (SSL) methods, both contrastive and non-contrastive, learn representations from different views or augmentations of inputs without human-annotated labels~\citep{chen2020simple, hjelm2018learning, wu2018unsupervised, tian2019contrastive, chen2021exploring, gao2021simcse, bachman2019learning, oord2018representation, ye2019unsupervised, henaff2020data, misra2020self, caron2020unsupervised, haochen2021provable, caron2021emerging,li2021self, zbontar2021barlow, tsai2021note, bardes2021vicreg, tian2020makes, robinson2021contrastive, dubois2022improving, tan2023contrastive, zhang2023matrix}. On downstream detection and segmentation tasks, self-supervised pretraining has outperformed its supervised counterpart~\citep{he2019momentum, DBLP:conf/cvpr/HeCXLDG22}. Masked Image Modeling (MIM) methods, including MAE and SimMIM~\citep{DBLP:conf/cvpr/HeCXLDG22, DBLP:conf/cvpr/Xie00LBYD022, DBLP:journals/corr/abs-2205-13515, ren2023deepmim, DBLP:conf/iclr/Bao0PW22, cao2022understand, tan2023information}, draw on the Masked Language Modeling (MLM) objective of BERT~\citep{devlin2018bert} to learn visual representations.

\paragraph{Vision-Language Models.}
Several vision-language models build on large language models (LLMs) for multimodal understanding. BLIP-2~\citep{li2023blip} connects frozen, off-the-shelf image encoders to large language models through a Querying Transformer trained in two stages, and reports state-of-the-art performance on several vision-language tasks with substantially fewer trainable parameters. BEiT-3~\citep{wang2022image} uses Multiway Transformers and unified pretraining on images, texts, and image-text pairs, and reports state-of-the-art results on both vision and vision-language tasks.

MiniGPT-4~\citep{zhu2023minigpt} aligns a visual encoder with the Vicuna LLM by training mainly a projection layer between them. It can generate detailed image descriptions, creative writing inspired by images, and instructional content based on images. Its results point to the importance of high-quality aligned datasets and show that this form of alignment is efficient to train.

\section{Conclusion}

The Gram contractions evaluate token and channel mixing exactly from second-order statistics. Under the stated dimensions, this takes $O(d^2+dd_v)$ forward work and summary state per causal token, with each output computed from its own prefix. Trace normalization bounds the spectral norm of each fixed Gram operator, and row normalization gives probability weights when the Gram entries are nonnegative and the row sums nonzero, as with strictly positive features.

The unmasked token operators are symmetric with rank at most $d$, whereas softmax weights can be asymmetric and full-rank. Three questions remain open: how this structural difference affects sequence modeling, how much accuracy the low-rank and windowed variants retain at their lower cost, and whether the matrix-exponential variant admits an efficient per-token recurrence. The first two call for empirical comparisons and approximation-error analysis.

\bibliographystyle{plainnat}
\bibliography{reference}

\end{document}